\documentclass[letterpaper,10pt,conference]{ieeeconf}
\IEEEoverridecommandlockouts
\usepackage{graphicx}
\usepackage{amsmath,amssymb,bm}
\usepackage{booktabs,array,tabularx}
\usepackage{algorithm}
\usepackage{algpseudocode}
\newtheorem{lemma}{Lemma}
\newtheorem{remark}{Remark}
\makeatletter
\def\@begintheorem#1#2{\@IEEEtmpitemindent\itemindent\topsep 0pt\rmfamily\trivlist%
    \item[\hskip\labelsep{\bfseries #1\ #2:}]\itemindent\@IEEEtmpitemindent}
\def\@opargbegintheorem#1#2#3{\@IEEEtmpitemindent\itemindent\topsep 0pt\rmfamily\trivlist%
    \item[\hskip\labelsep{\bfseries #1\ #2\ (#3):}]\itemindent\@IEEEtmpitemindent}
\makeatother
\renewenvironment{proof}{\par\noindent\textbf{Proof.}\ }{\hfill$\square$\par}
\usepackage[table]{xcolor}
\definecolor{rewardTask}{HTML}{E8F1FB}
\definecolor{rewardRegularization}{HTML}{EAF5EC}
\definecolor{rewardSafety}{HTML}{FFF0E0}
\usepackage{cite}
\usepackage{url}
\usepackage{flushend}
\usepackage{microtype}
\title{\LARGE\bf Proprioceptive Force Estimation for\\
Quadruped Locomotion and Human-Robot Interaction}

\author{Run Wang, Xu Yang, Alapati Tuerxun, Yilin Mo$^{\ast}$
\thanks{The authors are with the Department of Automation, Tsinghua University, Beijing, China.
        {\tt\small \{wangrun24, yangx21, alpttex25\}@mails.tsinghua.edu.cn, ylmo@tsinghua.edu.cn}}%
\thanks{$^{\ast}$Corresponding author}%
}

\begin{document}
\bstctlcite{icra:reference_style}
\maketitle
\thispagestyle{empty}
\pagestyle{empty}
\suppressfloats[t]

\begin{abstract}
    Payload forces must be accommodated during locomotion, while leash forces can specify desired motion. We investigate whether a shared three-dimensional force estimate in newtons, inferred from proprioceptive history under sustained loading, can support both tasks. An estimator and locomotion policy are jointly trained with supervised force and velocity outputs and learned latent context. The estimated force conditions locomotion and additionally generates planar-velocity and yaw-rate commands for leash guidance through an analytical map. In sustained-force simulation sweeps, temporal means of componentwise force root mean square error range from 1.44 to 2.83\,N. Compared with a domain-randomized baseline, the framework reduces velocity-tracking and base-orientation error scores by 21.6\% and 46.5\%, respectively, and increases mean survival from 68.29\% to 94.60\% in separate sustained-force tests. Unitree Go1 experiments demonstrate stationary vertical and horizontal force estimation, locomotion with an 8.5\,kg payload whose weight exceeds the 70\,N training force limit, and leash guidance using the same force-estimation interface.
\end{abstract}

\begin{figure}[t]
    \centering
    \includegraphics[width=\columnwidth]{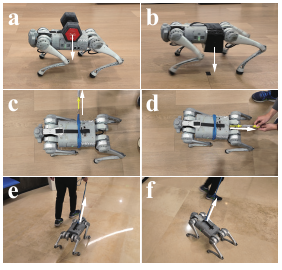}
    \vspace{-25pt}
    \caption{Go1 demonstrations of the external-force interface: (a) payload-force estimation; (b) payload locomotion using online force estimates; (c,d) horizontal-force estimation; and (e,f) leash-guided locomotion. Arrows indicate applied forces.}
    \label{fig:snapshots}
\end{figure}

\begin{figure*}[t]
    \centering
    \includegraphics[width=\textwidth]{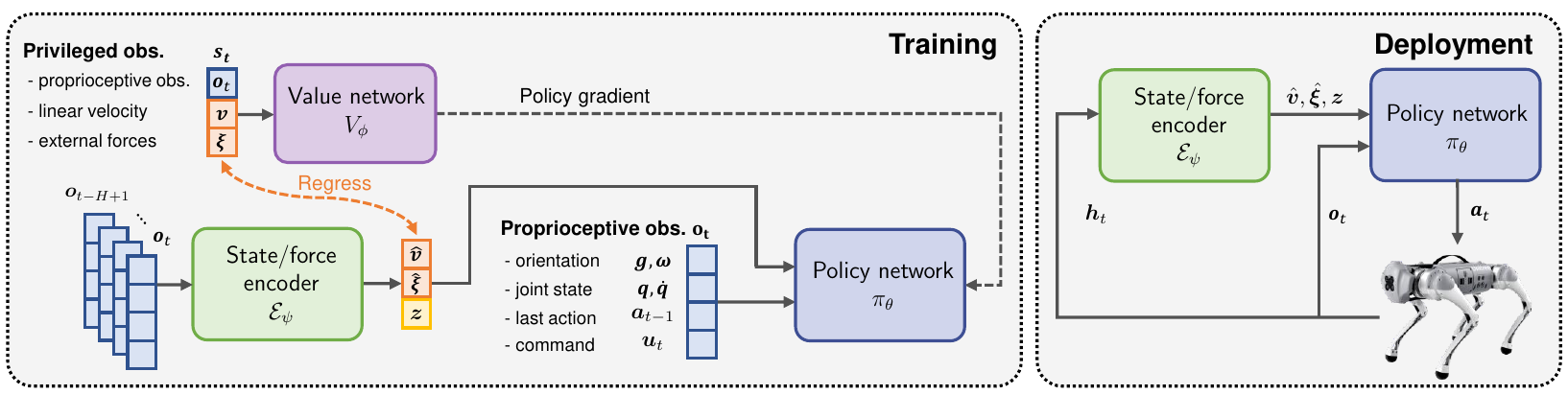}
    \caption{Shared force-estimation interface. Encoder $\mathcal{E}_{\psi}$ and policy $\pi_{\theta}$ are jointly trained with privileged critic $V_{\phi}$; normalization is omitted here. Force estimates condition locomotion in both tasks and additionally generate leash commands through $\mathcal{M}$. Payload commands are supplied externally. Deployed weights remain fixed.}
    \label{fig:overview}
\end{figure*}

\section{Introduction}

External forces play different roles in quadruped locomotion and physical interaction. A payload changes the loading that the robot must accommodate, while a leash pull can indicate a desired direction of travel. Learning-based controllers have enabled agile locomotion through simulation training and hardware deployment~\cite{hwangbo2019learning,miki2022learning,rudin2022learning}. Load carrying and physical guidance can further benefit from a force estimate usable by both locomotion and interaction controllers.

Domain randomization (DR) exposes a policy to varied dynamics to improve deployment robustness~\cite{peng2018sim,aljalbout2026reality}, and simulation parameter distributions can be refined using real-world rollouts~\cite{chebotar2019closing}. History-based adaptation infers context from recent observations and actions~\cite{kumar2021rma,nahrendra2023dreamwaq,long2023hybrid}. These representations support locomotion without assigning physical meaning to every coordinate, but do not necessarily provide a force readout for downstream control.

Prior work has jointly trained force estimation and locomotion for tug-guided navigation~\cite{defazio2023seeing}. We ask whether a force estimate recovered from proprioception under sustained loading can serve as a shared physical interface for payload locomotion and guidance. We supervise three-axis base-frame force with simulator-applied loads in newtons. For payload locomotion, the estimate conditions the policy under supplied velocity commands; for leash guidance, it additionally generates planar velocity and yaw-rate commands through an analytical mapping. Table~\ref{tab:qualitative_comparison} positions these two uses relative to existing deployment representations.

Force inference must distinguish persistent loading from contact and commanded motion after the initial response subsides. Our estimator predicts force, velocity, and latent context from a short observation history. Supervised regression and reinforcement learning jointly train the estimator and locomotion policy; their outputs update online with fixed network weights.

We evaluate this shared interface through force-reference measurements, force-conditioned locomotion, and force-driven guidance on a Unitree Go1 (Fig.~\ref{fig:snapshots}). Our contributions are:
\begin{itemize}
    \item A proprioceptive estimation and locomotion framework trained under sustained simulator-applied forces, with explicit supervision of three-axis force outputs in newtons.
    \item A shared force interface that conditions payload locomotion and analytically generates continuous-valued planar-velocity and yaw-rate commands for leash guidance.
    \item Evidence of sustained-force estimation in simulation and on a stationary Go1, with online estimates supporting load carrying beyond the training force magnitude and leash guidance.
\end{itemize}

\section{Related Work}
\label{sec:related}

\subsection{Proprioceptive Adaptation for Locomotion}
\label{sec:related_adaptation}

Privileged-policy distillation~\cite{lee2020learning} and Rapid Motor Adaptation (RMA)~\cite{kumar2021rma} infer latent context from proprioceptive history. DreamWaQ combines velocity estimation with implicit terrain context~\cite{nahrendra2023dreamwaq}, while the Hybrid Internal Model learns response embeddings through contrastive learning~\cite{long2023hybrid}. Ji et al. jointly learn a state estimator and locomotion policy~\cite{ji2022concurrent}.

The top group of Table~\ref{tab:qualitative_comparison} contrasts latent context with explicit physical outputs. Bridging Adaptivity and Safety (BAS) estimates payload mass, center-of-mass (CoM) shift, and friction for locomotion and safety assessment~\cite{zhong2025bas}. Beyond Robustness estimates payload mass, friction, position, and velocity to stabilize moving loads~\cite{chang2025beyond}; Rapid Embodiment Adaptation identifies joint limits and mass properties for cross-embodiment locomotion~\cite{li2026rapid}. Our framework retains history-based conditioning but supervises the resultant three-axis force without reconstructing the full payload state.

\subsection{Physical Identification and Adaptive Control}
\label{sec:related_identification}

Classical adaptive control updates robot and payload parameters online~\cite{slotine1987adaptive}. Quadruped methods recover trunk mass and CoM recursively~\cite{tournois2017online}, estimate payload disturbances for inverse-dynamics control~\cite{jin2022unknown}, or update model-predictive control using Kalman-filter mass and CoM estimates~\cite{haack2025adaptive}. Gu et al. combine static-payload CoM identification with an adaptive human-arm model~\cite{gu2025adaptive}. Other methods identify terrain friction and stiffness~\cite{Chen24physical} or calibrate simulator parameters~\cite{sobanbabu2025sampling}. Momentum observers estimate external disturbances for whole-body compensation~\cite{morlando2021wholebody}. Our force output instead serves as a time-varying input to both locomotion and guidance.

\subsection{Force Estimation for Physical Interaction}
\label{sec:related_force}

Model-based external-wrench observers combine dynamics, inertial measurement units (IMUs), and contact-force measurements~\cite{benallegue2018model}. DeFazio et al. jointly train force estimation and locomotion using imposed base-velocity changes as tug proxies; estimated peaks select discrete left/right navigation choices~\cite{defazio2023seeing}. HAC-LOCO supervises velocity and force heads from proprioceptive history to condition locomotion, and learns residual velocity commands for compliant response~\cite{zhou2025hac}. FACET distills force-responsive impedance tracking without exposing a calibrated force output~\cite{xu2025facet}. Model-based human--quadruped cooperation uses measured interaction forces to adjust footholds for guided transport~\cite{gu2025hierarchical}.

Table~\ref{tab:qualitative_comparison} separates the deployed representation from its locomotion and guidance roles. Its lower group shows that an explicit force estimate alone does not distinguish our approach: HAC-LOCO also provides one. Our distinction is the shared use of force components supervised in newtons under sustained loading: they condition locomotion under supplied commands and analytically generate continuous planar-velocity and yaw-rate commands for guidance. Section~\ref{sec:method} defines this interface, and Section~\ref{sec:experiments} evaluates its readout and both control uses.

\begin{table*}[t]
    \centering
    \caption{Main deployment representations and their uses in locomotion and force interaction. Auxiliary velocity and latent outputs are omitted for force estimators.}
    \label{tab:qualitative_comparison}
    \footnotesize
    \setlength{\tabcolsep}{5pt}
    \renewcommand{\arraystretch}{1.0}
    \begin{tabularx}{\textwidth}{>{\raggedright\arraybackslash}p{0.17\textwidth}>{\raggedright\arraybackslash}p{0.27\textwidth}>{\raggedright\arraybackslash}p{0.20\textwidth}>{\raggedright\arraybackslash}X}
        \toprule
        \textbf{Method}                          & \textbf{Online representation}  & \textbf{Locomotion role} & \textbf{Guidance / compliance}          \\
        \midrule
        \rowcolor{rewardTask}
        \multicolumn{4}{l}{\textit{Locomotion adaptation}}                                                                                              \\
        RMA~\cite{kumar2021rma}                  & Latent context                  & Policy conditioning      & ---                                     \\
        BAS~\cite{zhong2025bas}                  & Mass, CoM shift, friction       & Policy + safety          & ---                                     \\
        Beyond Robustness~\cite{chang2025beyond} & Load state and properties       & Load stabilization       & ---                                     \\
        \midrule
        \rowcolor{rewardTask}
        \multicolumn{4}{l}{\textit{Force-based interaction}}                                                                                            \\
        DeFazio et al.~\cite{defazio2023seeing}  & Tug estimate ($\Delta v$ proxy) & Policy conditioning      & Discrete turn selection                 \\
        FACET~\cite{xu2025facet}                 & Latent context                  & Impedance tracking       & Tunable compliance                      \\
        Gu et al.~\cite{gu2025hierarchical}      & Measured force                  & Cooperative transport    & Force-guided footholds                  \\
        HAC-LOCO~\cite{zhou2025hac}              & 3-D force estimate              & Force conditioning       & \textbf{Learned} velocity residuals     \\
        \textbf{Ours}                            & 3-D force estimate (N)          & Force conditioning       & \textbf{Analytic} velocity/yaw-rate map \\
        \bottomrule
    \end{tabularx}
    \par\smallskip
    \parbox{\textwidth}{\scriptsize $\Delta v$: imposed velocity change used for training; ---: no dedicated force-interaction interface.}
\end{table*}

\section{Method}
\label{sec:method}

\subsection{Force Representation and Modeling Assumptions}
\label{sec:method_scope}

Let $\bm{\xi}=[f_x,f_y,f_z]^\top$ denote the resultant external force applied to the base, expressed in newtons along its forward, left, and upward axes. For a payload of mass $m_p$ supported only by the robot, let $\bm{c}_p^W$ denote its center-of-mass position in the world frame, $\bm{g}^W$ the gravitational acceleration, and $R_{WB}$ the rotation from base to world coordinates. The force exerted by the payload on the robot is
\begin{equation}
    \bm{\xi}_p=m_pR_{WB}^{\top}
    \left(\bm{g}^{W}-\ddot{\bm{c}}_p^{W}\right),
    \label{eq:payload_force}
\end{equation}
where $\ddot{\bm{c}}_p^W$ is the payload acceleration. This representation omits payload-induced moments and changes in inertia.

For stationary loads, $\bm{\xi}_p\approx m_pR_{WB}^{\top}\bm{g}^{W}$; base pitch and roll give gravity horizontal components. During motion, payload acceleration also contributes, so mass alone does not specify the instantaneous force.

Proprioceptive force inference can be confounded by contact, terrain, and actuation. We assume forces persist long enough to produce distinguishable histories within the training distribution; unique identification is not guaranteed. Finite history and robot response limit estimation bandwidth.

Control performance depends on both estimation accuracy and the learned policy response, motivating separate evaluations of force error and locomotion.

\subsection{Proprioceptive Force Estimation}

Figure~\ref{fig:overview} shows the architecture. At step $t$, let $\bm{\omega}_t$ be base angular velocity, $\bm{g}_t$ projected unit gravity, and $\bm{q}_t,\dot{\bm{q}}_t$ joint positions and velocities. The observation also includes the previous action $\bm{a}_{t-1}$ and command $\bm{u}_t=[v_x^{\mathrm{cmd}},v_y^{\mathrm{cmd}},\omega_z^{\mathrm{cmd}}]_t^\top$:
\begin{equation}
    \bm{o}_t=[\bm{\omega}_t;\bm{g}_t;\bm{q}_t;\dot{\bm{q}}_t;\bm{a}_{t-1};\bm{u}_t].
\end{equation}
Semicolons denote vector concatenation. Angular velocity and projected gravity are expressed in the base frame.

We stack $H=10$ observations, newest first, into the history $\bm{h}_t=[\bm{o}_t;\ldots;\bm{o}_{t-H+1}]$. Each observation has 45 coordinates, yielding a 450-dimensional history spanning 0.18\,s at 50\,Hz. A multilayer perceptron (MLP) encoder $\mathcal{E}_{\psi}$, parameterized by $\psi$, predicts base linear velocity $\bm{v}_t$, external force $\bm{\xi}_t$, and a ten-dimensional latent representation $\bm{z}_t$:
\begin{equation}
    (\hat{\bar{\bm{v}}}_t,\hat{\bar{\bm{\xi}}}_t,\bm{z}_t)
    =\mathcal{E}_{\psi}(\bar{\bm{h}}_t).
    \label{eq:estimator}
\end{equation}
Hats denote estimates and bars denote empirical normalization (Section~\ref{sec:joint}). Linear velocity is expressed in the base frame; latent coordinates have no assigned physical meaning.

Previous actions and commands help distinguish maneuvers from loading responses, while $\bm{z}_t$ represents context beyond the supervised outputs.

\subsection{Force-Conditioned Locomotion}

The actor $\pi_{\theta}$, parameterized by $\theta$, conditions its action on the instantaneous observation and encoder outputs:
\begin{equation}
    \bm{a}_t\sim\pi_{\theta}(\cdot\mid\bar{\bm{o}}_t,\hat{\bar{\bm{v}}}_t,\hat{\bar{\bm{\xi}}}_t,\bm{z}_t).
\end{equation}
The actor receives 61 inputs and produces a 12-dimensional action $\bm{a}_t$. Desired joint angles are $\bm{q}^{\mathrm{des}}_t=\bm{q}^{\mathrm{nominal}}+0.2\bm{a}_t$, where $\bm{q}^{\mathrm{nominal}}$ contains nominal hip, thigh, and calf angles of $(0,1,-2)$ radians on each leg. Joint proportional--derivative (PD) control tracks these targets with torque saturation at the model's actuator limits.

The training-only asymmetric critic $V_{\phi}$, parameterized by $\phi$, estimates state value~\cite{pinto2017asymmetric}. Its input is the normalized privileged state $\bm{s}_t=[\bm{o}_t;\bm{v}_t;\bm{\xi}_t]$ (51 coordinates), with simulator-provided velocity and force. All three networks have hidden widths $[512,256,128]$, exponential linear unit (ELU) activations, and linear outputs.

\subsection{Joint Optimization}
\label{sec:joint}
We jointly train the estimator and policy following~\cite{ji2022concurrent}. The objective combines a proximal policy optimization (PPO) loss $\mathcal{L}_{\mathrm{PPO}}$~\cite{schulman2017proximal} with a supervised estimation loss $\mathcal{L}_{\mathrm{adapt}}$, weighted by $c$:
\begin{equation}
    \mathcal{L}=\mathcal{L}_{\mathrm{PPO}}+c\mathcal{L}_{\mathrm{adapt}},
    \qquad c=3.
    \label{eq:loss}
\end{equation}
PPO contains policy, value, and entropy terms. Targets $\bm{y}=[\bm{v};\bm{\xi}]$ receive uniform label noise with half-widths 0.1\,m/s and 2\,N, respectively, and are normalized componentwise:
\begin{equation}
    \mathcal{N}_y(\bm{y})=(\bm{y}-\bm{\mu}_y)\oslash(\bm{\sigma}_y+10^{-2}),
\end{equation}
where $\bm{\mu}_y,\bm{\sigma}_y$ are running means and standard deviations, and $\oslash$ is elementwise division. Targets, current observations, histories, and critic inputs use separate statistics, frozen after $2\times10^7$ samples. For a minibatch of $B$ samples indexed by $b$, noisy labels $\bm{y}^{\mathrm{lbl}}$, and $\hat{\bar{\bm{y}}}=[\hat{\bar{\bm{v}}};\hat{\bar{\bm{\xi}}}]$,
\begin{equation}
    \mathcal{L}_{\mathrm{adapt}}=\frac{1}{6B}\sum_{b=1}^{B}
    \left\|\hat{\bar{\bm{y}}}_b-\mathcal{N}_y(\bm{y}^{\mathrm{lbl}}_b)\right\|_2^2.
\end{equation}
Regression supervises the six velocity and force outputs, while the PPO policy loss propagates gradients through all 16 encoder outputs. The critic uses privileged inputs independently of the encoder. The policy receives estimated quantities throughout training, and one adaptive moment estimation (Adam) optimizer~\cite{kingma2014adam} updates all network parameters $\Theta=(\psi,\theta,\phi)$ (Algorithm~\ref{alg:training}).

\begin{algorithm}[tb]
    \caption{Joint Force Estimation and Locomotion Learning}
    \label{alg:training}
    \small
    \begin{algorithmic}[1]
        \Require Randomized simulator; $\mathcal{E}_\psi$, $\pi_\theta$, critic $V_\phi$
        \State Initialize $\Theta=(\psi,\theta,\phi)$ and normalization statistics
        \For{each training iteration}
        \State Collect rollouts using $\mathcal{E}_\psi$ and $\pi_\theta$, updating normalization statistics until frozen
        \State Store normalized inputs, targets, and transitions
        \State Compute returns and advantages using $V_\phi$
        \For{each PPO minibatch}
        \State Recompute encoder, actor, and critic outputs
        \State $\mathcal{L}\gets\mathcal{L}_{\mathrm{PPO}}+c\mathcal{L}_{\mathrm{adapt}}$
        \State $\Theta\gets\operatorname{Adam}(\Theta,\nabla_{\Theta}\mathcal{L})$
        \EndFor
        \EndFor
        \State Save $\mathcal{E}_\psi$, $\pi_\theta$, and normalization statistics
    \end{algorithmic}
\end{algorithm}

\subsection{Deployment and Interaction Interface}

At 50\,Hz, the history is updated and the estimator and actor mean are evaluated with fixed weights and frozen normalizers. The policy receives normalized estimates; $\mathcal{N}_y^{-1}$ recovers physical velocity and force:
\begin{equation}
    \begin{bmatrix}\hat{\bm{v}}_t\\\hat{\bm{\xi}}_t\end{bmatrix}
    =\mathcal{N}_y^{-1}\left(
    \begin{bmatrix}\hat{\bar{\bm{v}}}_t\\\hat{\bar{\bm{\xi}}}_t\end{bmatrix}
    \right).
    \label{eq:physical_readout}
\end{equation}
Payload locomotion uses online force estimates and externally supplied motion commands. For leash guidance, let $\hat{\bm{f}}=[\hat f_x,\hat f_y]^\top$ be the estimated horizontal force. For nonzero $\hat{\bm{f}}$, its base-frame direction is $\alpha(\hat{\bm{f}})=\operatorname{atan2}(\hat f_y,\hat f_x)\in(-\pi,\pi]$. The command map $\mathcal{M}$ uses a radial dead zone with threshold $F_d=10$\,N:
\begin{equation}
    \mathcal{M}(\hat{\bm{f}})=
    \begin{cases}
        \bm{0}, & \|\hat{\bm{f}}\|_2\le F_d, \\
        [K_1\hat f_x;K_2\hat f_y;K_3\alpha(\hat{\bm{f}})],
                & \|\hat{\bm{f}}\|_2>F_d.
    \end{cases}
    \label{eq:leash}
\end{equation}
The output comprises planar velocity $\bm{v}_{xy}^{\mathrm{cmd}}=[v_x^{\mathrm{cmd}},v_y^{\mathrm{cmd}}]^\top$ and yaw rate $\omega_z^{\mathrm{cmd}}$. The gains $K_1,K_2$ have units of $\mathrm{m}\,\mathrm{s}^{-1}\,\mathrm{N}^{-1}$, and $K_3$ has units of $\mathrm{s}^{-1}$.

These command sources implement the two control roles in Table~\ref{tab:qualitative_comparison}. The estimator does not classify loading as a payload or a human; we evaluate these uses in separate hardware trials. The following analysis relates force error to command error where the map is active.

\begin{samepage}
    \begin{lemma}[Local force-to-command sensitivity]
        \label{lem:command_error}
        Let $\bm{f}=[f_x,f_y]^\top$ be the true horizontal force and $\bm{e}=\hat{\bm{f}}-\bm{f}$ its estimation error. For an error budget $\varepsilon_f>0$ and a forward-force lower bound $F_*>F_d+\varepsilon_f$, assume $\|\bm{e}\|_2\le\varepsilon_f$ and $f_x\ge F_*$. The differences $\Delta\bm{v}_{xy}^{\mathrm{cmd}}$ and $\Delta\omega_z^{\mathrm{cmd}}$ between the commands $\mathcal{M}(\hat{\bm{f}})$ and $\mathcal{M}(\bm{f})$ satisfy
        \begin{align}
            \|\Delta\bm{v}_{xy}^{\mathrm{cmd}}\|_2
             & \le \max(|K_1|,|K_2|)\varepsilon_f,\nonumber      \\
            |\Delta\omega_z^{\mathrm{cmd}}|
             & \le \frac{|K_3|\varepsilon_f}{F_*-\varepsilon_f}.
            \label{eq:command_error_bound}
        \end{align}
    \end{lemma}
    \par
\end{samepage}

\begin{proof}
    Every point on the segment joining $\bm{f}$ and $\hat{\bm{f}}$ has forward component at least $F_*-\varepsilon_f>F_d$. The map is therefore active and the angle continuous along the segment. The planar bound follows from the operator norm of $\operatorname{diag}(K_1,K_2)$. Since $\|\nabla\alpha\|_2=1/\|\bm{f}\|_2$, integrating the gradient along the segment yields the angular bound.
\end{proof}

\begin{remark}[Interaction gain selection]
    \label{rem:sensitivity_design}
    For forward pulls with $f_x\ge22$\,N and an assumed error budget $\varepsilon_f=2$\,N, Lemma~\ref{lem:command_error} gives command-error tolerances of 0.1\,m/s and 0.5\,rad/s when $\max(|K_1|,|K_2|)\le0.05$\,m\,s$^{-1}$\,N$^{-1}$ and $|K_3|\le5$\,s$^{-1}$, matching the deployed gains. This is a local sensitivity bound under the assumed error budget.
\end{remark}

\subsection{Reward and Domain Randomization}
\label{sec:reward_training}

The shared reward is $r_t=\max(0,\Delta t\sum_jw_j r_j)$, where $\Delta t=0.02$\,s is the control interval and $r_j,w_j$ are the components and weights in Table~\ref{tab:reward}. Tracking, posture, effort, and contact terms follow standard locomotion practice~\cite{rudin2022learning,nahrendra2023dreamwaq,long2023hybrid}. A curriculum scales both planar command ranges from 0.5 to 1.0, giving forward limits from $\pm0.75$ to $\pm1.5$\,m/s and lateral limits from $\pm0.5$ to $\pm1.0$\,m/s.

\begin{table}[t]
    \centering
    \small
    \setlength{\tabcolsep}{3pt}
    \renewcommand{\arraystretch}{1.1}
    \caption{Reward Function Components.}
    \label{tab:reward}
    \begin{tabular}{l|l|l}
        \hline
        \rowcolor{rewardTask}
        \textbf{Task rewards}     & \textbf{Expression ($r_j$)}                                                              & \textbf{Weight ($w_j$)} \\
        \hline
        Tracking lin. vel.        & $\exp\left\{-4 \left\| \bm{v}_{xy}^{\text{cmd}} - \bm{v}_{xy} \right\|_2^2\right\}$      & $1.0$                   \\
        Tracking ang. vel.        & $\exp\left\{-4 \left( \omega_{z}^{\text{cmd}} - \omega_{z} \right)^2\right\}$            & $0.5$                   \\
        \hline
        \rowcolor{rewardRegularization}
        \textbf{Regularization}   & \textbf{Expression ($r_j$)}                                                              & \textbf{Weight ($w_j$)} \\
        \hline
        Alive                     & $1.0$                                                                                    & $0.5$                   \\
        Base height               & $\left( h_b^{\text{nominal}} - h_b \right)^2$                                            & $-50.0$                 \\
        Orientation               & $\left\| \bm{g}_{xy} \right\|_2^2$                                                       & $-5.0$                  \\
        Foot position             & $\sum_{i} \left\| \bm{p}_{xy,i} - \bm{p}_{xy,i}^{\text{nominal}} \right\|_2^2$           & $-10.0$                 \\
        Lin. vel. (z)             & $v_{z}^2$                                                                                & $-2.0$                  \\
        Ang. vel. (xy)            & $\left\| \bm{\omega}_{xy} \right\|_2^2$                                                  & $-0.05$                 \\
        Action rate               & $\left\| \bm{a}_t - \bm{a}_{t-1} \right\|_2^2$                                           & $-0.02$                 \\
        Torques                   & $\left\| \bm{\tau} \right\|_2^2$                                                         & $-2.5 \times 10^{-4}$   \\
        Joint acceleration        & $\left\| \ddot{\bm{q}} \right\|_2^2$                                                     & $-2.5 \times 10^{-7}$   \\
        Stand still               & $\left\| \bm{q}^{\text{nominal}} - \bm{q} \right\|_1 \cdot \mathbb{I}_{\text{zero cmd}}$ & $-1.0$                  \\
        Feet air time             & $\mathbb{I}_{\text{move cmd}} \sum_i (T_i-0.5)d_i$                                       & $0.5$                   \\
        \hline
        \rowcolor{rewardSafety}
        \textbf{Safety penalties} & \textbf{Expression ($r_j$)}                                                              & \textbf{Weight ($w_j$)} \\
        \hline
        Foot slip                 & $\sum_i \left\|c_i \cdot \dot{\bm{p}}_i\right\|_2$                                       & $-0.2$                  \\
        Joint position limit      & $\left\| \bm{q}_{\text{viol}} \right\|_1$                                                & $-10.0$                 \\
        Collision                 & $\sum_{i} \mathbb{I}_{\text{collision on body i}}$                                       & $-1.0$                  \\
        \hline
    \end{tabular}
\end{table}

Here $h_b$ is base height, with nominal value $h_b^{\mathrm{nominal}}=0.26$\,m. For foot $i$, $\bm{p}_i$ and $\dot{\bm{p}}_i$ denote position and velocity; $\bm{p}_{xy,i}^{\mathrm{nominal}}$ is the nominal planar position in a gravity-aligned, base-centered frame with the robot's yaw. The vector $\bm{\tau}$ contains joint torques and $\ddot{\bm{q}}$ joint accelerations. Contact indicator $c_i$ is one above 1\,N vertical foot force; $d_i$ marks the first contact after flight, using current and previous contact states, and $T_i$ is the corresponding air time in seconds. Indicators $\mathbb{I}_{\text{move cmd}}$ and $\mathbb{I}_{\text{zero cmd}}$ select planar command norms above and below 0.1\,m/s. The vector $\bm{q}_{\mathrm{viol}}$ measures violations of 90\% soft joint limits; the collision indicator is one for a body contact exceeding 0.1\,N.

Training combines sustained base forces, friction randomization, and intermittent velocity pushes using Table~\ref{tab:dr}. Here $\mathcal{U}(a,b)$ is uniform between $a$ and $b$, and $\mathbb{S}^2$ is the unit sphere of force directions.

\begin{table}[t]
    \centering
    \caption{Domain randomization shared by all compared controllers.}
    \label{tab:dr}
    \small
    \setlength{\tabcolsep}{3pt}
    \renewcommand{\arraystretch}{1.1}
    \begin{tabular}{l|l}
        \hline
        \rowcolor{rewardTask}
        \textbf{Quantity}           & \textbf{Setting}          \\
        \hline
        Friction coefficient        & $\mathcal{U}(0.5,1.25)$   \\
        Velocity-push interval      & $15$\,s                   \\
        Push velocity magnitude     & $\mathcal{U}(0,1)$\,m/s   \\
        Continuous-force resampling & Every $11$\,s             \\
        Continuous-force magnitude  & $\mathcal{U}(0,70)$\,N    \\
        Continuous-force direction  & Uniform on $\mathbb{S}^2$ \\
        \hline
    \end{tabular}
\end{table}

\begin{figure}[t]
    \centering
    \includegraphics[width=\columnwidth]{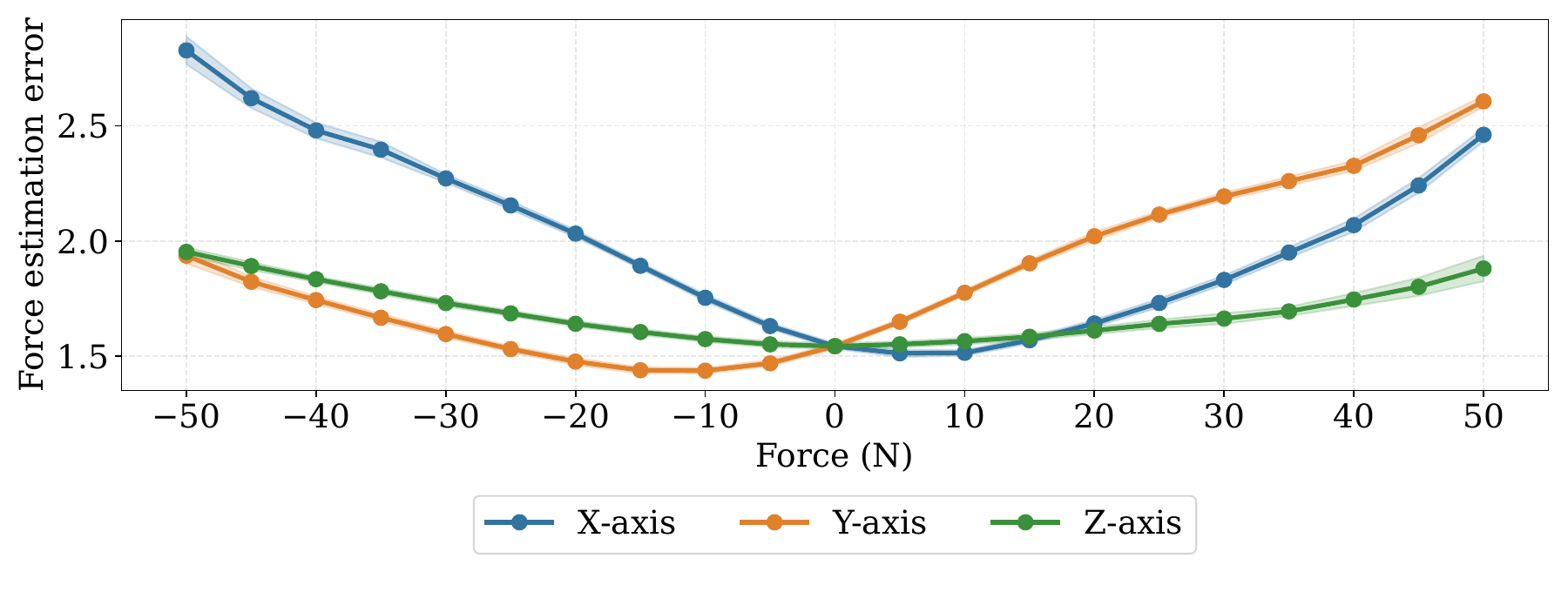}
    \caption{Simulation force-estimation errors under sustained base forces along each axis. Curves and bands show the mean $\pm$ one temporal standard deviation, as defined in Section~\ref{sec:metrics}. All three force components enter each score.}
    \label{fig:force_sim}
\end{figure}

\begin{figure}[t]
    \centering
    \includegraphics[width=\columnwidth]{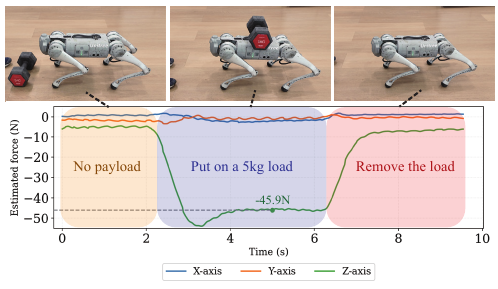}
    \caption{Stationary hardware force estimation. A 5\,kg dumbbell is placed on the robot and removed. The curves show estimated force components, with payload gravity serving as the reference for the settled vertical value.}
    \label{fig:force_real}
\end{figure}

\begin{figure}[t]
    \centering
    \includegraphics[width=\columnwidth]{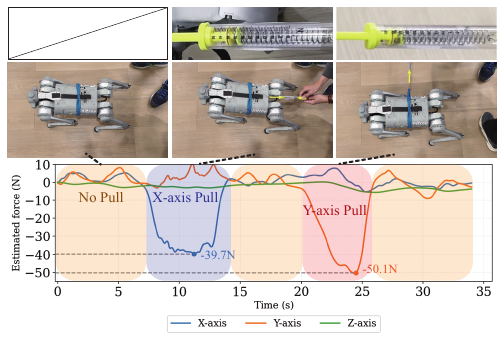}
    \caption{Stationary horizontal-force estimation. Spring-scale magnitudes are 41\,N at $t=11$\,s for the x-axis pull and 45\,N at $t=24$\,s for the y-axis pull. The corresponding marked estimates are $-39.7$\,N and $-50.1$\,N; signs indicate the negative base-axis directions.}
    \label{fig:force_xy}
\end{figure}

\begin{figure*}[t]
    \centering
    \includegraphics[width=\textwidth]{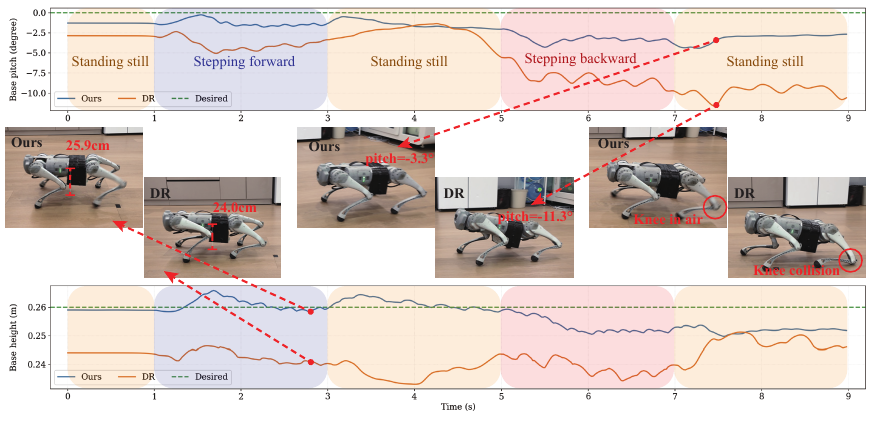}
    \caption{Hardware locomotion with an 8.5\,kg payload secured beneath the torso. Ours uses online force estimates from the proprioceptive estimator to condition the locomotion policy. Both controllers receive the same velocity commands.}
    \label{fig:loco_real}
\end{figure*}

\section{Experiments}
\label{sec:experiments}

Following Table~\ref{tab:qualitative_comparison}, we evaluate force-readout accuracy (\textbf{Q1}), locomotion performance and robustness under loading (\textbf{Q2}), and force-driven command generation for human guidance (\textbf{Q3}).

\subsection{Simulation Setup}
\label{sec:setup}

Policies are trained in Isaac Gym~\cite{makoviychuk2021isaac} using Legged Gym~\cite{rudin2022learning} with 4,096 parallel environments on flat terrain. All controllers share reward weights, domain randomization (Table~\ref{tab:dr}), and network hidden widths. PPO training uses Adam with a learning rate adapted according to Kullback--Leibler (KL) divergence and initialized at $10^{-3}$. Training runs for 2,000 iterations, taking approximately one hour on an RTX 4090.

\subsection{Evaluation Protocol}
\label{sec:protocol}

We compare four controllers: \textbf{DR}, which uses only the current observation; \textbf{implicit adaptation (IA)}, a history baseline based on a variational autoencoder (VAE) and inspired by~\cite{nahrendra2023dreamwaq}; \textbf{Ours}, the proposed framework; and \textbf{Oracle}, which receives ground-truth force and velocity. IA reconstructs the current observation with KL regularization weight $\beta=1$ and has no explicit velocity head. The comparison evaluates complete controllers with different adaptation modules; Oracle is an empirical privileged reference.

Force-estimation and tracking sweeps use 4,096 robots, and survival tests use 1,024. Evaluation uses flat terrain with friction 0.7, no added mass or inertia, and no observation noise or impulsive pushes. Rollouts last 30\,s, with velocity commands resampled every 5\,s unless standing is specified. Each method uses a single trained checkpoint. A fall or overturn triggers a reset. Error sweeps allow resets; survival requires uninterrupted rollouts.

Hardware demonstrations use a Go1 on flat indoor surfaces. The estimator and policy run at 50\,Hz on an onboard Jetson Nano through a pybind11 interface~\cite{jakobpybind,yang2023cajun}, with joint proportional gain $K_p=30$ and derivative gain $K_d=1$.

Stationary force trials use payload gravity and pointwise spring-scale readings as references. Payload trials use identical supplied commands for both controllers; leash trials use force-generated commands.

\begin{figure}[t]
    \centering
    \includegraphics[width=\columnwidth]{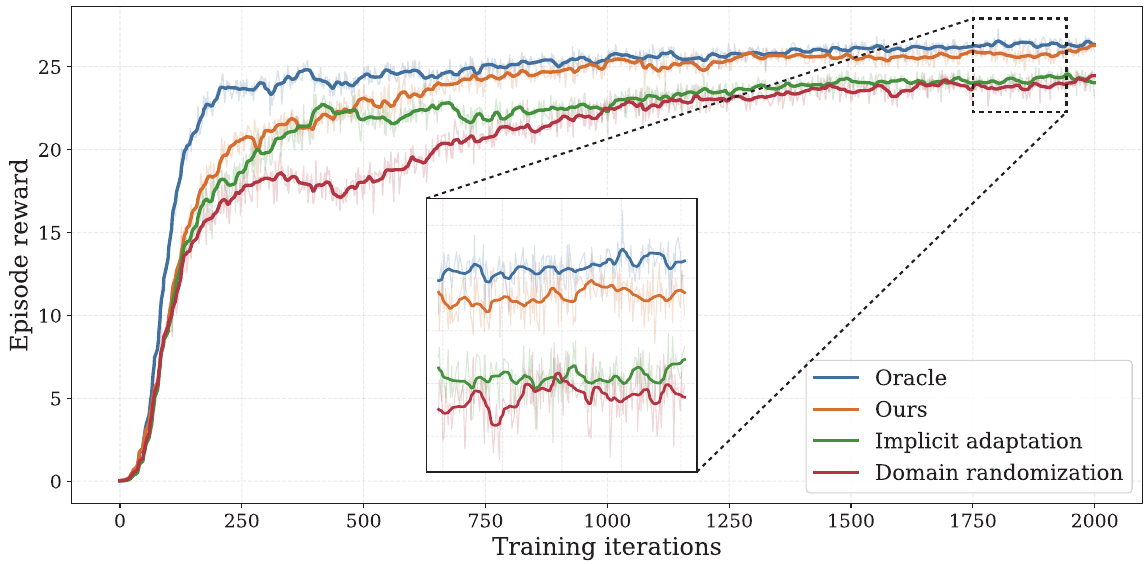}
    \caption{Recorded training rewards. Solid curves use Gaussian smoothing with standard deviation $\sigma_{\mathrm{smooth}}=3$ recorded samples; lighter traces are the unsmoothed recorded rewards. Oracle uses ground-truth conditioning.}
    \label{fig:training}
\end{figure}

\begin{figure}[t]
    \centering
    \includegraphics[width=\columnwidth]{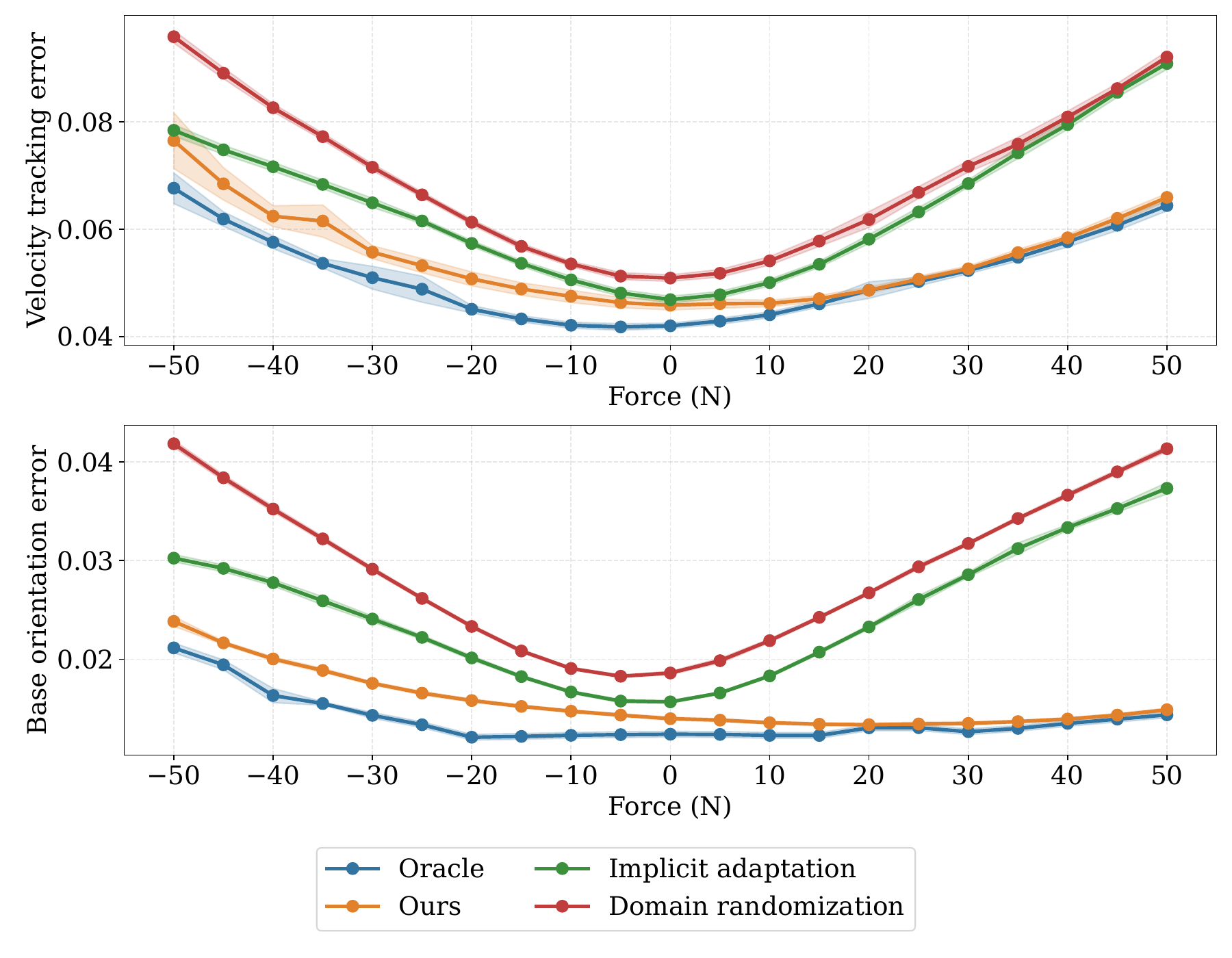}
    \caption{Dimensionless velocity-tracking and base-orientation error scores under external x-axis forces. Oracle uses ground-truth conditioning.}
    \label{fig:loco_sim}
\end{figure}

\begin{figure*}[t]
    \centering
    \includegraphics[width=\textwidth]{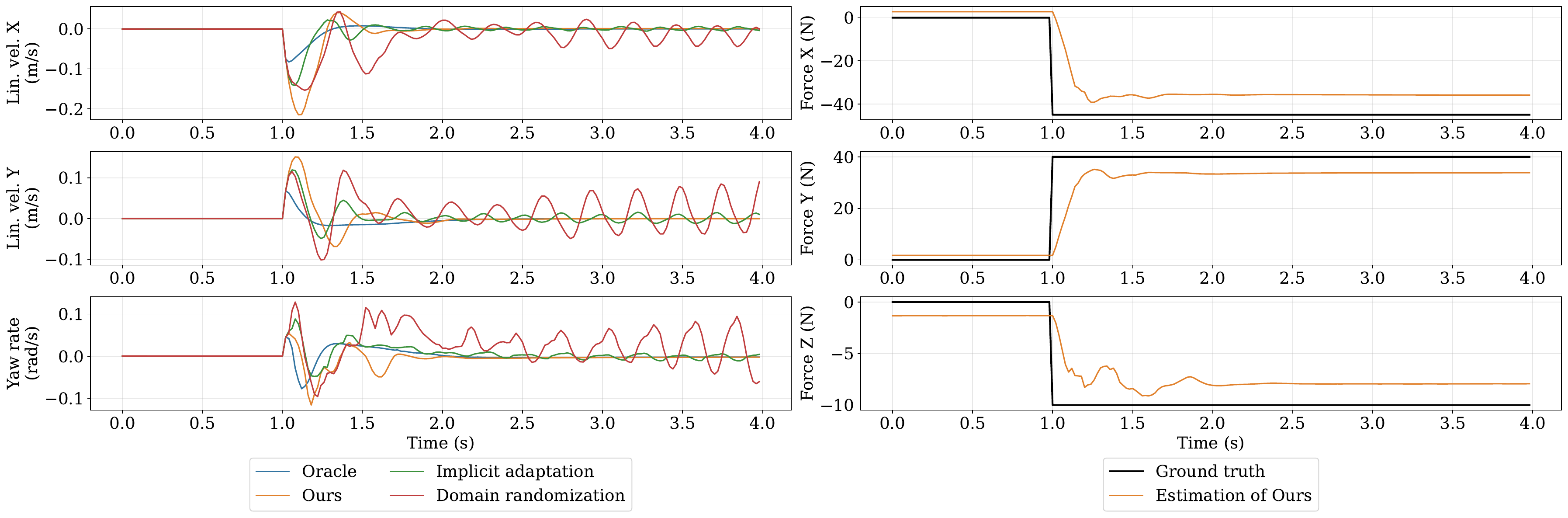}
    \caption{Simulated transient response to a force applied from $t=1$\,s. Left: planar velocity and yaw rate. Right: applied and estimated force.}
    \label{fig:transient}
\end{figure*}

\begin{figure*}[t]
    \centering
    \includegraphics[width=\textwidth]{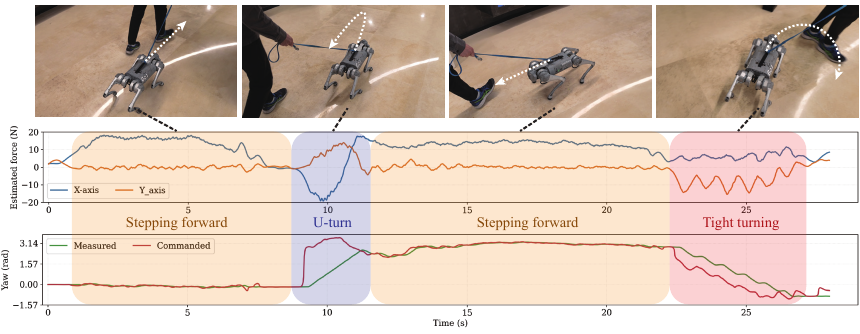}
    \caption{Hardware leash guidance using online force estimates. The sequence includes forward motion, a $180^\circ$ turn, and small-radius turning. The heading reference is reconstructed offline from force estimates and measured yaw.}
    \label{fig:leash}
\end{figure*}

\subsection{Metrics}
\label{sec:metrics}
Let $\delta\bm{v}_{xy}=\bm{v}_{xy}^{\mathrm{cmd}}-\bm{v}_{xy}$ and $\delta\omega_z=\omega_z^{\mathrm{cmd}}-\omega_z$ denote planar velocity and yaw-rate tracking errors. We define a dimensionless composite score using reference scales $v_{\mathrm{ref}}=1$\,m/s and $\omega_{\mathrm{ref}}=1$\,rad/s. At each recorded time $t$, errors are first aggregated over $N$ robots, indexed by $i$:
\begin{align}
    d_{\mathrm{vel},t}   & =\sqrt{\frac{1}{3N}\sum_i
        \left(\frac{\|\delta\bm{v}_{xy,i,t}\|_2^2}{v_{\mathrm{ref}}^2}
    +\frac{(\delta\omega_{z,i,t})^2}{\omega_{\mathrm{ref}}^2}\right)},\label{eq:vel}       \\
    d_{\mathrm{ori},t}   & =\sqrt{\frac{1}{N}\sum_i\|\bm{g}_{xy,i,t}\|_2^2},\label{eq:ori} \\
    d_{\mathrm{force},t} & =\sqrt{\frac{1}{3N}\sum_i
        \|\hat{\bm{\xi}}_{i,t}-\bm{\xi}_{i,t}\|_2^2}.
\end{align}
Orientation error is dimensionless; force error is a componentwise root mean square error (RMSE) in newtons.

For all simulation error summaries, we remove temporal outliers using a standard interquartile range (IQR) filter with a factor of 1.5, applied separately to each method, force condition, and metric. Force-estimation curves show temporal means and standard deviations.

For each metric, the overall score $E$ is the square root of the mean squared per-time errors pooled across force conditions. The relative reduction is $1-E_{\mathrm{ours}}/E_{\mathrm{base}}$, where the subscripts identify Ours and the compared baseline. Survival is the fraction of robots completing 30\,s without falling or overturning; its relative improvement is based on the mean rate over the three axis conditions.

\subsection{Q1: External-Force Estimation}

In simulation, sustained forces are applied along each base axis from $-50$ to $50$\,N in 5\,N increments, within the 70\,N training range. Temporal means of componentwise force RMSE range from 1.44 to 2.83\,N, including errors on unloaded axes (Fig.~\ref{fig:force_sim}).

On hardware, a 5\,kg dumbbell is placed on the stationary Go1 for approximately 4\,s and removed (Fig.~\ref{fig:force_real}). The near-steady vertical estimate is $-45.9$\,N, differing from the $-49$\,N gravitational reference by 3.1\,N (6.3\%); horizontal estimates remain near zero. The approximately 1\,s response after each change reflects the combined robot--estimator dynamics, beyond the 0.18\,s history span.

Separate pulls along the negative base $x$ and $y$ axes provide horizontal references (Fig.~\ref{fig:force_xy}). At $t=11$\,s, the x-axis estimate is $-39.7$\,N against a signed spring-scale reference of $-41$\,N, an absolute error of 1.3\,N (3.2\%). At $t=24$\,s, the y-axis estimate is $-50.1$\,N against $-45$\,N, an error of 5.1\,N (11.3\%).

\subsection{Q2: Locomotion Performance and Robustness}

\subsubsection{Tracking and Payload Locomotion}

In the recorded training runs, all methods improve rapidly at first (Fig.~\ref{fig:training}). Ours approaches the Oracle reward trace, while DR and IA reach lower rewards.

The tracking sweep applies x-axis forces from $-50$ to $50$\,N in 5\,N increments (Fig.~\ref{fig:loco_sim}). Ours reduces the pooled tracking and orientation scores by 21.6\% and 46.5\% relative to DR, and by 15.1\% and 37.0\% relative to IA. These gains characterize the complete force-conditioned framework.

The hardware trial uses an 8.5\,kg payload, approximately 70\% of nominal robot mass, secured beneath the torso during standing, forward stepping, and backward stepping (Fig.~\ref{fig:loco_real}). Around $t=3$\,s, the displayed heights are 24\,cm for DR and 25.9\,cm for Ours; around $t=7.5$\,s, pitches are $-11.3^\circ$ and $-3.3^\circ$, respectively. Hind-knee ground contacts occur in the DR backward-stepping sequence and are absent in the displayed sequence for Ours.

The payload weighs approximately 83\,N, exceeding the 70\,N training range, and changes mass distribution and inertia. This trial demonstrates online force-conditioned load carrying beyond the trained applied-force model.

\subsubsection{Survival and Transient Response}
\begin{table}[t]
    \centering
    \caption{Simulation survival rates at the specified force conditions. Oracle is a privileged reference; bold marks the best among Ours, IA, and DR.}
    \label{tab:survival}
    \small
    \setlength{\tabcolsep}{3pt}
    \renewcommand{\arraystretch}{1.1}
    \begin{tabular}{l|r|r|r}
        \hline
        \rowcolor{rewardSafety}
        \textbf{Method} & \textbf{X: 80\,N} & \textbf{Y: 60\,N} & \textbf{Z: 80\,N} \\
        \hline
        Oracle          & 100.00\%          & 100.00\%          & 84.08\%           \\
        \hline
        Ours            & \textbf{100.00\%} & \textbf{100.00\%} & \textbf{83.79\%}  \\
        IA              & 52.25\%           & 84.77\%           & 69.04\%           \\
        DR              & 94.73\%           & 44.04\%           & 66.11\%           \\
        \hline
    \end{tabular}
\end{table}

Survival tests apply continuous forces of 80\,N along $x$, 60\,N along $y$, or 80\,N along $z$, with a random sign per robot (Table~\ref{tab:survival}). Commands specify standing in the $x/y$ tests and are resampled every 5\,s in the $z$ test. Ours has the highest survival among the proprioceptive controllers and remains close to Oracle. Mean survival is 94.60\% for Ours and 68.29\% for DR, an increase of 26.30 percentage points or 38.5\%. This corresponds to 2,906 and 2,098 surviving episodes across 3,072 starts. The 80\,N cases exceed the training force magnitude.

In the transient test, a force is applied from $t=1$\,s while the robot is commanded to stand (Fig.~\ref{fig:transient}). Ours and Oracle return toward the stationary command, while DR and IA show larger oscillations. Estimated components follow the applied force after the onset transient.

\subsection{Q3: Leash-Guided Control}

An elastic tether attached to the robot's back transmits human guidance forces. The same force output that conditions locomotion generates motion commands through~\eqref{eq:leash}, with $K_1=K_2=0.05$\,m\,s$^{-1}$\,N$^{-1}$ and $K_3=5$\,s$^{-1}$. The robot follows the pulling direction through forward locomotion, a $180^\circ$ turn, and small-radius turning (Fig.~\ref{fig:leash}).

The plotted heading reference is reconstructed offline as measured yaw plus $\operatorname{atan2}(\hat f_y,\hat f_x)$ when horizontal force magnitude exceeds 6\,N, and measured yaw otherwise. It is unwrapped and bidirectionally exponentially smoothed with coefficient 0.5.

\section{Discussion and Limitations}
\label{sec:discussion}

The experiments support two uses of a shared force output: conditioning locomotion under prescribed commands and generating commands for human guidance. The reported gains characterize the jointly trained framework as a whole. Hardware demonstrations use a Go1 on flat terrain, with force-reference measurements obtained during stationary loading. Further work will examine component contributions and force estimation during dynamic interactions, and extend evaluation across terrains and payload configurations.

\section{Conclusion}

We presented a shared proprioceptive force interface for payload locomotion and leash guidance. Sustained-force supervision provides a three-axis estimate in newtons that conditions locomotion and generates interaction commands. The complete framework improves the reported simulation tracking and survival metrics. Go1 trials demonstrate stationary force readout, online load carrying beyond the training force magnitude, and force-driven guidance.

\section*{Acknowledgment}

OpenAI Codex (GPT-6) assisted with manuscript organization and language editing. The experimental results originate from the authors' existing experiments.

\begingroup
\microtypesetup{expansion=false}
\bibliographystyle{IEEEtran}
\bibliography{ICRA_references}

@article{slotine1987adaptive,
  title   = {On the adaptive control of robot manipulators},
  author  = {Slotine, Jean-Jacques E. and Li, Weiping},
  journal = {Int. J. Robot. Res.},
  volume  = {6},
  number  = {3},
  pages   = {49--59},
  year    = {1987},
  doi     = {10.1177/027836498700600303}
}

@inproceedings{zhong2025bas,
  author    = {Zhong, Yichao and Zhang, Chong and He, Tairan and Shi, Guanya},
  title     = {Bridging Adaptivity and Safety: Learning Agile Collision-Free Locomotion Across Varied Physics},
  booktitle = {Proc. Learn. Dyn. Control Conf. (L4DC)},
  volume    = {283},
  pages     = {1498--1511},
  year      = {2025},
  publisher = {PMLR}
}

@inproceedings{chang2025beyond,
  author    = {Chang, Leixin and Nai, Yuxuan and Chen, Hua and Yang, Liangjing},
  title     = {Beyond Robustness: Learning Unknown Dynamic Load Adaptation for Quadruped Locomotion on Rough Terrain},
  booktitle = {Proc. IEEE Int. Conf. Robot. Autom. (ICRA)},
  pages     = {10282--10288},
  year      = {2025},
  doi       = {10.1109/ICRA55743.2025.11128639}
}

@inproceedings{haack2025adaptive,
  author    = {Haack, Jonas and Stark, Franek and Vyas, Shubham and Kirchner, Frank and Kumar, Shivesh},
  title     = {Adaptive Model-Based Control of Quadrupeds via Online System Identification using {Kalman} Filter},
  booktitle = {Proc. IEEE/RSJ Int. Conf. Intell. Robots Syst. (IROS)},
  pages     = {5039--5044},
  year      = {2025},
  doi       = {10.1109/IROS60139.2025.11246753}
}

@inproceedings{tournois2017online,
  author    = {Tournois, Guido and Focchi, Michele and Del Prete, Andrea and Orsolino, Romeo and Caldwell, Darwin G. and Semini, Claudio},
  title     = {Online Payload Identification for Quadruped Robots},
  booktitle = {Proc. IEEE/RSJ Int. Conf. Intell. Robots Syst. (IROS)},
  pages     = {4889--4896},
  year      = {2017},
  doi       = {10.1109/IROS.2017.8206367}
}

@article{jin2022unknown,
  author  = {Jin, Bingchen and Ye, Shusheng and Su, Juntong and Luo, Jianwen},
  title   = {Unknown Payload Adaptive Control for Quadruped Locomotion With Proprioceptive Linear Legs},
  journal = {IEEE/ASME Trans. Mechatronics},
  volume  = {27},
  number  = {4},
  pages   = {1891--1899},
  year    = {2022},
  doi     = {10.1109/TMECH.2022.3170548}
}

@article{morlando2021wholebody,
  author  = {Morlando, Viviana and Teimoorzadeh, Ainoor and Ruggiero, Fabio},
  title   = {Whole-Body Control With Disturbance Rejection Through a Momentum-Based Observer for Quadruped Robots},
  journal = {Mech. Mach. Theory},
  volume  = {164},
  pages   = {104412},
  year    = {2021},
  doi     = {10.1016/j.mechmachtheory.2021.104412}
}

@inproceedings{zhou2025hac,
  author    = {Zhou, Xiang and Zhang, Xinyu and Wu, Tong and Zhang, Qingrui and Zhang, Lixian},
  title     = {{HAC-LOCO}: Learning Hierarchical Active Compliance Control for Quadruped Locomotion Under Continuous External Disturbances},
  booktitle = {Proc. IEEE/RSJ Int. Conf. Intell. Robots Syst. (IROS)},
  pages     = {10649--10655},
  year      = {2025},
  doi       = {10.1109/IROS60139.2025.11247775}
}

@inproceedings{xu2025facet,
  author    = {Xu, Botian and Weng, Haoyang and Lu, Qingzhou and Gao, Yang and Xu, Huazhe},
  title     = {{FACET}: Force-Adaptive Control via Impedance Reference Tracking for Legged Robots},
  booktitle = {Proc. Conf. Robot Learn. (CoRL)},
  volume    = {305},
  pages     = {705--720},
  year      = {2025},
  publisher = {PMLR}
}

@article{gu2025hierarchical,
  author  = {Gu, Sai and Meng, Fei and Liu, Botao and Chen, Xuechao and Yu, Zhangguo and Huang, Qiang},
  title   = {Hierarchical Cooperative Locomotion Control of Human and Quadruped Robot Based on Interactive Force Guidance},
  journal = {IEEE/ASME Trans. Mechatronics},
  volume  = {30},
  number  = {6},
  pages   = {6677--6687},
  year    = {2025},
  doi     = {10.1109/TMECH.2025.3535721}
}

@article{gu2025adaptive,
  author  = {Gu, Sai and Meng, Fei and Liu, Botao and Chen, Xuechao and Yu, Zhangguo and Huang, Qiang},
  title   = {Adaptive Interactive Control of Human and Quadruped Robot Load Motion},
  journal = {IEEE/ASME Trans. Mechatronics},
  volume  = {30},
  number  = {2},
  pages   = {1459--1470},
  year    = {2025},
  doi     = {10.1109/TMECH.2024.3425857}
}

@misc{li2026rapid,
  author        = {Li, Dichen and Ai, Bo and Bohlinger, Nico and Peters, Jan and Su, Hao and Christensen, Henrik I.},
  title         = {Rapid Embodiment Adaptation for Quadrupedal Locomotion},
  year          = {2026},
  eprint        = {2608.01506},
  archivePrefix = {arXiv},
  primaryClass  = {cs.RO},
  doi           = {10.48550/arXiv.2608.01506}
}

@article{hwangbo2019learning,
  title     = {Learning agile and dynamic motor skills for legged robots},
  author    = {Hwangbo, Jemin and Lee, Joonho and Dosovitskiy, Alexey and Bellicoso, Dario and Tsounis, Vassilios and Koltun, Vladlen and Hutter, Marco},
  journal   = {Sci. Robot.},
  volume    = {4},
  number    = {26},
  pages     = {eaau5872},
  year      = {2019},
  doi       = {10.1126/scirobotics.aau5872},
  publisher = {American Association for the Advancement of Science}
}

@article{miki2022learning,
  title     = {Learning robust perceptive locomotion for quadrupedal robots in the wild},
  author    = {Miki, Takahiro and Lee, Joonho and Hwangbo, Jemin and Wellhausen, Lorenz and Koltun, Vladlen and Hutter, Marco},
  journal   = {Sci. Robot.},
  volume    = {7},
  number    = {62},
  pages     = {eabk2822},
  year      = {2022},
  doi       = {10.1126/scirobotics.abk2822},
  publisher = {American Association for the Advancement of Science}
}

@inproceedings{rudin2022learning,
  title     = {Learning to walk in minutes using massively parallel deep reinforcement learning},
  author    = {Rudin, Nikita and Hoeller, David and Reist, Philipp and Hutter, Marco},
  booktitle = {Proc. Conf. Robot Learn. (CoRL)},
  volume    = {164},
  pages     = {91--100},
  year      = {2022},
  publisher = {PMLR}
}

@inproceedings{kumar2021rma,
  title     = {{RMA}: Rapid motor adaptation for legged robots},
  author    = {Kumar, Ashish and Fu, Zipeng and Pathak, Deepak and Malik, Jitendra},
  booktitle = {Proc. Robot.: Sci. Syst. (RSS)},
  year      = {2021},
  doi       = {10.15607/RSS.2021.XVII.011}
}

@inproceedings{nahrendra2023dreamwaq,
  author    = {Aswin Nahrendra, I Made and Yu, Byeongho and Myung, Hyun},
  booktitle = {Proc. IEEE Int. Conf. Robot. Autom. (ICRA)},
  title     = {{DreamWaQ}: Learning Robust Quadrupedal Locomotion With Implicit Terrain Imagination via Deep Reinforcement Learning},
  year      = {2023},
  pages     = {5078--5084},
  doi       = {10.1109/ICRA48891.2023.10161144}
}

@inproceedings{long2023hybrid,
  title     = {Hybrid Internal Model: Learning Agile Legged Locomotion with Simulated Robot Response},
  author    = {Long, Junfeng and Wang, Zirui and Li, Quanyi and Gao, Jiawei and Cao, Liu and Pang, Jiangmiao},
  booktitle = {Proc. Int. Conf. Learn. Represent. (ICLR)},
  year      = {2024}
}

@article{ji2022concurrent,
  title     = {Concurrent training of a control policy and a state estimator for dynamic and robust legged locomotion},
  author    = {Ji, Gwanghyeon and Mun, Juhyeok and Kim, Hyeongjun and Hwangbo, Jemin},
  journal   = {IEEE Robot. Autom. Lett.},
  volume    = {7},
  number    = {2},
  pages     = {4630--4637},
  year      = {2022},
  doi       = {10.1109/LRA.2022.3151396},
  publisher = {IEEE}
}

@inproceedings{benallegue2018model,
  title        = {Model-based external force/moment estimation for humanoid robots with no torque measurement},
  author       = {Benallegue, Mehdi and Gergondet, Pierre and Audren, Herv\'{e} and Mifsud, Alexis and Morisawa, Mitsuharu and Lamiraux, Florent and Kheddar, Abderrahmane and Kanehiro, Fumio},
  booktitle    = {Proc. IEEE Int. Conf. Robot. Autom. (ICRA)},
  pages        = {3122--3129},
  doi          = {10.1109/ICRA.2018.8460809},
  year         = {2018},
  organization = {IEEE}
}

@inproceedings{defazio2023seeing,
  title     = {Seeing-eye quadruped navigation with force responsive locomotion control},
  author    = {DeFazio, David and Hirota, Eisuke and Zhang, Shiqi},
  booktitle = {Proc. Conf. Robot Learn. (CoRL)},
  volume    = {229},
  pages     = {2184--2194},
  year      = {2023},
  publisher = {PMLR}
}

@inproceedings{pinto2017asymmetric,
  author    = {Pinto, Lerrel and Andrychowicz, Marcin and Welinder, Peter and Zaremba, Wojciech and Abbeel, Pieter},
  year      = {2018},
  booktitle = {Proc. Robot.: Sci. Syst. (RSS)},
  title     = {Asymmetric Actor Critic for Image-Based Robot Learning},
  doi       = {10.15607/RSS.2018.XIV.008}
}

@misc{schulman2017proximal,
  title        = {Proximal policy optimization algorithms},
  author       = {Schulman, John and Wolski, Filip and Dhariwal, Prafulla and Radford, Alec and Klimov, Oleg},
  howpublished = {arXiv:1707.06347},
  year         = {2017},
  doi          = {10.48550/arXiv.1707.06347}
}

@inproceedings{makoviychuk2021isaac,
  author    = {Makoviychuk, Viktor and Wawrzyniak, Lukasz and Guo, Yunrong and Lu, Michelle and Storey, Kier and Macklin, Miles and Hoeller, David and Rudin, Nikita and Allshire, Arthur and Handa, Ankur and State, Gavriel},
  booktitle = {Proc. Neural Inf. Process. Syst. Track Datasets Benchmarks},
  title     = {{Isaac Gym}: High Performance {GPU} Based Physics Simulation For Robot Learning},
  volume    = {1},
  year      = {2021}
}

@inproceedings{kingma2014adam,
  author    = {Kingma, Diederik P. and Ba, Jimmy},
  title     = {{Adam}: A Method for Stochastic Optimization},
  booktitle = {Proc. Int. Conf. Learn. Represent. (ICLR)},
  year      = {2015}
}

@misc{jakobpybind,
  title        = {{pybind/pybind11}: Version 3.0.0 (final)},
  author       = {Jakob, Wenzel and Schreiner, Henry and Rhinelander, Jason and Grosse-Kunstleve, Ralf W. and Moldovan, Dean and Smirnov, Ivan and Gokaslan, Aaron and Jadoul, Yannick and Huebl, Axel and Staletic, Boris and Izmailov, Sergei and Cousineau, Eric and Carlstrom, Michael and Merry, Bruce and Spicuzza, Dustin and Lee, Antony and Corlay, Sylvain and Burns, Lori A. and {Dan} and Pan, Xuehai and {bennorth} and Houliston, Trent and Lyskov, Sergey and {b-pass} and {jbarlow} and Haschke, Robert and {\v{S}im\'{a}\v{c}ek}, Michael and Steinberg, Ethan},
  howpublished = {Zenodo, computer software},
  year         = {2025},
  version      = {3.0.0},
  doi          = {10.5281/zenodo.15857181},
  url          = {https://doi.org/10.5281/zenodo.15857181}
}

@inproceedings{yang2023cajun,
  title     = {{CAJun}: Continuous adaptive jumping using a learned centroidal controller},
  author    = {Yang, Yuxiang and Shi, Guanya and Meng, Xiangyun and Yu, Wenhao and Zhang, Tingnan and Tan, Jie and Boots, Byron},
  booktitle = {Proc. Conf. Robot Learn. (CoRL)},
  volume    = {229},
  pages     = {2791--2806},
  year      = {2023},
  publisher = {PMLR}
}

@inproceedings{peng2018sim,
  title        = {Sim-to-real transfer of robotic control with dynamics randomization},
  author       = {Peng, Xue Bin and Andrychowicz, Marcin and Zaremba, Wojciech and Abbeel, Pieter},
  booktitle    = {Proc. IEEE Int. Conf. Robot. Autom. (ICRA)},
  pages        = {3803--3810},
  year         = {2018},
  doi          = {10.1109/ICRA.2018.8460528},
  organization = {IEEE}
}

@article{lee2020learning,
  title     = {Learning quadrupedal locomotion over challenging terrain},
  author    = {Lee, Joonho and Hwangbo, Jemin and Wellhausen, Lorenz and Koltun, Vladlen and Hutter, Marco},
  journal   = {Sci. Robot.},
  volume    = {5},
  number    = {47},
  pages     = {eabc5986},
  year      = {2020},
  doi       = {10.1126/scirobotics.abc5986},
  publisher = {American Association for the Advancement of Science}
}

@article{Chen24physical,
  author  = {Chen, Jiaqi and Frey, Jonas and Zhou, Ruyi and Miki, Takahiro and Martius, Georg and Hutter, Marco},
  journal = {IEEE Robot. Autom. Lett.},
  title   = {Identifying Terrain Physical Parameters From Vision - Towards Physical-Parameter-Aware Locomotion and Navigation},
  year    = {2024},
  volume  = {9},
  number  = {11},
  pages   = {9279--9286},
  doi     = {10.1109/LRA.2024.3455788}
}

@inproceedings{chebotar2019closing,
  author    = {Chebotar, Yevgen and Handa, Ankur and Makoviychuk, Viktor and Macklin, Miles and Issac, Jan and Ratliff, Nathan and Fox, Dieter},
  title     = {Closing the Sim-to-Real Loop: Adapting Simulation Randomization with Real World Experience},
  booktitle = {Proc. IEEE Int. Conf. Robot. Autom. (ICRA)},
  pages     = {8973--8979},
  year      = {2019},
  doi       = {10.1109/ICRA.2019.8793789}
}

@article{aljalbout2026reality,
  author  = {Aljalbout, Elie and Xing, Jiaxu and Romero, Angel and Akinola, Iretiayo and Garrett, Caelan Reed and Heiden, Eric and Gupta, Abhishek and Hermans, Tucker and Narang, Yashraj and Fox, Dieter and Scaramuzza, Davide and Ramos, Fabio},
  title   = {The Reality Gap in Robotics: Challenges, Solutions, and Best Practices},
  journal = {Annu. Rev. Control Robot. Auton. Syst.},
  volume  = {9},
  number  = {1},
  pages   = {403--432},
  year    = {2026},
  doi     = {10.1146/annurev-control-031924-100130}
}

@inproceedings{sobanbabu2025sampling,
  author    = {Sobanbabu, Nikhil and He, Guanqi and He, Tairan and Yang, Yuxiang and Shi, Guanya},
  title     = {Sampling-based System Identification with Active Exploration for Legged {Sim2Real} Learning},
  booktitle = {Proc. Conf. Robot Learn. (CoRL)},
  volume    = {305},
  pages     = {578--598},
  year      = {2025},
  publisher = {PMLR}
}

@ieeetranbstctl{icra:reference_style,
  ctluse_forced_etal       = {yes},
  ctlmax_names_forced_etal = {6},
  ctlnames_show_etal       = {1}
}
\endgroup
\end{document}